\documentclass[10pt,twocolumn,letterpaper]{article}

\usepackage[pagenumbers]{cvpr} 

\usepackage{graphicx}
\usepackage{amsmath}
\usepackage{amssymb}
\usepackage{amsmath, amssymb, amsthm}
\usepackage{booktabs}
\usepackage{multirow}
\usepackage{cancel}
\usepackage{amssymb,amsmath,amsthm}
\usepackage{xpatch}

\AtBeginDocument{\xpatchcmd\proof{$\blacksquare$}{\qedsymbol}{}{}}
\renewcommand{\qedsymbol}{}
\newtheorem{theorem}{Theorem}
\newtheorem{lemma}{Lemma}

\DeclareMathOperator*{\argmin}{arg\,min}

\usepackage[pagebackref,breaklinks,colorlinks]{hyperref}

\usepackage[capitalize]{cleveref}
\crefname{section}{Sec.}{Secs.}
\Crefname{section}{Section}{Sections}
\Crefname{table}{Table}{Tables}
\crefname{table}{Tab.}{Tabs.}

\def\cvprPaperID{14} 
\def\confName{CVPR}
\def\confYear{2022}
\usepackage{xcolor}

\begin{document}

\title{RODD: A Self-Supervised Approach for Robust Out-of-Distribution Detection}

\author{Umar Khalid, Ashkan Esmaeili, Nazmul Karim, and Nazanin Rahnavard \\
Department of Electrical and Computer Engineering\\
University of Central Florida, USA\\
{\tt\small umarkhalid@knights.ucf.edu, ashkan.esmaeili@ucf.edu, nazmul.karim18@knights.ucf.edu,} \\
{\tt\small nazanin.rahnavard@ucf.edu
}
 }

\maketitle

\begin{abstract}
\vspace{-1mm}
Recent studies have started to address the concern of detecting and rejecting the out-of-distribution (OOD) samples as a major challenge in the safe deployment of deep learning (DL) models. It is desired that the DL model should only be confident about the in-distribution (ID) data which reinforces the driving principle of the OOD detection. In this paper, we propose a simple yet effective generalized OOD detection method independent of out-of-distribution datasets. Our approach relies on self-supervised feature learning of the training samples, where the embeddings lie on a compact low-dimensional space. Motivated by the recent studies that show self-supervised adversarial contrastive learning helps robustify the model, we empirically show that a pre-trained model with self-supervised contrastive learning  yields a better model for uni-dimensional feature learning in the latent space. The method proposed in this work, referred to as \texttt{RODD}, outperforms SOTA detection performance on extensive suite of benchmark datasets on OOD detection tasks. On the CIFAR-100 benchmarks, \texttt{RODD} achieves a  26.97 $\%$ lower false positive rate (FPR@95) compared to SOTA methods. Our code is publicly available.\footnote {\url{https://github.com/UmarKhalidcs/RODD}}

\end{abstract}

\vspace{-5mm}
\section{Introduction}
In a real-world deployment, machine learning models are generally exposed to the \emph{out-of-distribution} (OOD) objects that they have not experienced during the training. Detecting such OOD samples is of paramount importance in safety-critical applications such as health-care and autonomous driving~\cite{filos2020can,khalid2022rf}.
Therefore, the researchers have started to address the issue of OOD detection more recently \cite{luan2021out,vyas2018out,jeong2020ood,chen2021atom,hsu2020generalized,huang2021mos,besnier2021triggering,zaeemzadeh2021out}. Most of the recent studies \cite{hendrycks2018deep,yu2019unsupervised,lee2017training,lee2018simple} on OOD detection use OOD data for the model regularization such that some distance metric between the ID and OOD distributions is maximized. In recent studies~\cite{pidhorskyi2018generative,shalev2018out}, generative models and auto-encoders have been proposed to tackle OOD detection. However, they require OOD samples for hyper-parameter tuning. In the real-world scenarios, OOD detectors  are distribution-agnostic. To overcome this limitation, some other methods that are independent of OOD data during the training process have been proposed~\cite{du2022vos,huang2021mos,sun2021react,zaeemzadeh2021out,yoshihashi2019classification,hsu2020generalized, joneidi2020select}. Such methods either use the membership probabilities \cite{du2022vos,hsu2020generalized,sun2021react,huang2021mos} or a feature embedding \cite{zaeemzadeh2021out,yoshihashi2019classification} to calculate an uncertainty score.  In \cite{yoshihashi2019classification}, the authors proposed to reconstruct the samples to produce a discriminate feature space. Similarly, \cite{du2022vos} proposed synthesizing virtual outliers to regularize the model’s decision boundary during training.
Nevertheless, the performance of the methods that rely on either reconstruction or generation \cite{yoshihashi2019classification,du2022vos,pidhorskyi2018generative} degrades on large-scale datasets or video classification scenarios.\\
\indent In this work, we claim that if the feature vectors belonging to each known class lie on a low-dimensional subspace, a representative singular vector can be calculated for each class that can be used to calculate uncertainty scores \cite{zaeemzadeh2021out}. In order to achieve such a compact representation of the features belonging to each class, we have leveraged contrastive learning as a pre-training tool that has improved the performance of the proposed robust out-of-distribution detector (\texttt{RODD}) as it has helped the better feature mapping in the latent space during the downstream fine-tuning stage~\cite{xue2022investigating,karim2021adversarial}. Self-supervised pre-training, where we use adversaries as a form of data augmentation, helps to raise the \texttt{RODD}'s performance in the settings with corrupted samples. This concept has been established by~\cite{chen2020adversarial,jiang2020robust,kim2020adversarial,ho2020contrastive,yi2021improved}  that a self-supervised contrastive adversarial learning can generate an adversarially robust model during the fine-tuning. The overall architecture of the \texttt{RODD} is shown in Fig. \ref{Methodfig}.
\begin{figure*}[t!]
\centering
\includegraphics[height=10cm, width=16cm,trim={.5cm 0.1cm 1.5cm 0cm}, clip]{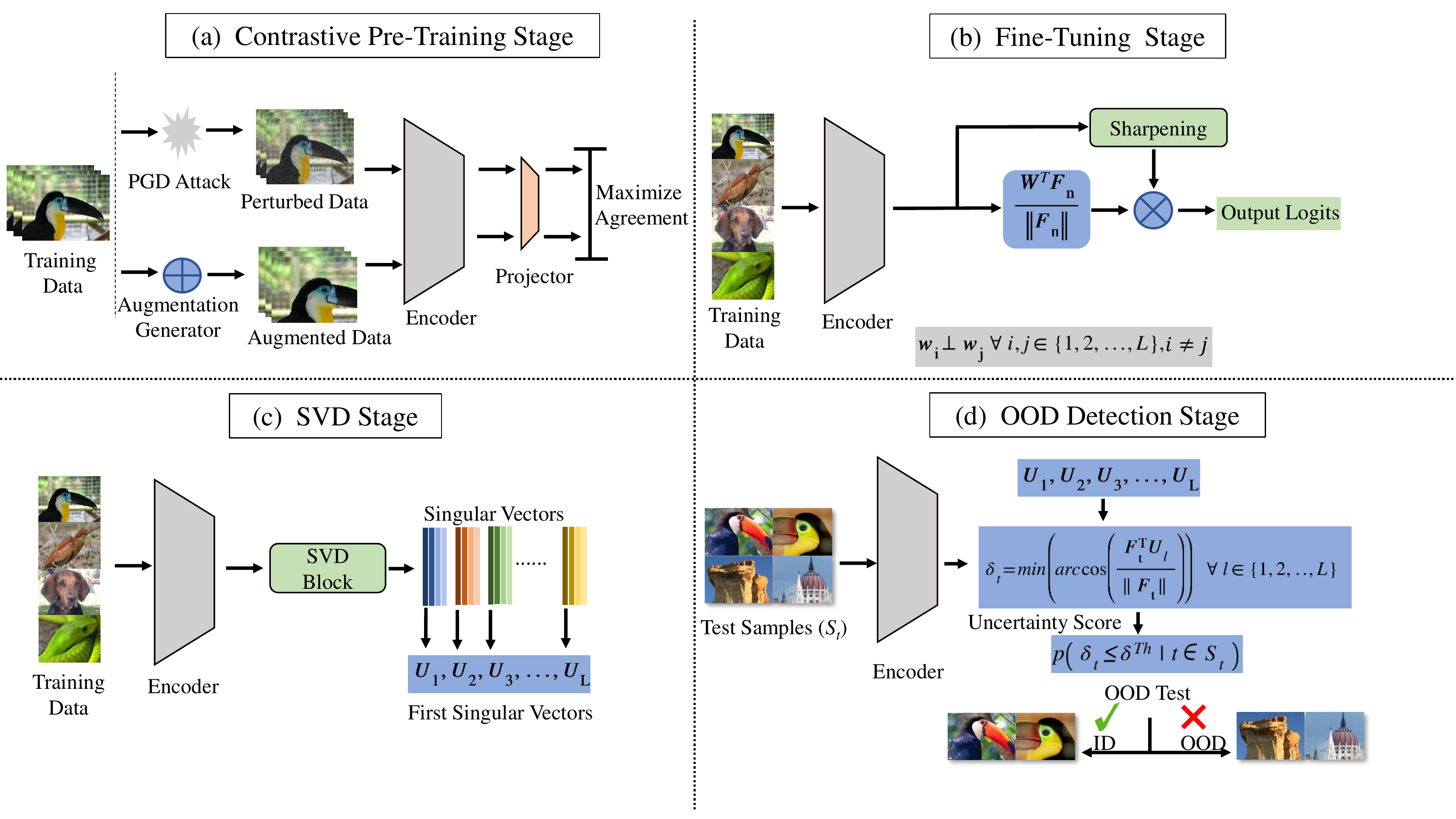}
\caption{\footnotesize{ Overall architecture of the proposed OOD detection method. \textcolor{red}{(a)} In the first step, self-supervised adversarial contrastive learning is performed. \textcolor{red} {(b)} Secondly, the encoder is fine-tuned by freezing the weights ($\mathbf W$) of the penultimate layer. The columns of $\mathbf W$ are initialized to be orthonormal.\textcolor{red}{(c)} Thirdly, employing singular value decomposition (SVD), we calculate the first singular vector of each class using its features. \textcolor{red}{(d)} The final step is the OOD detection, where an uncertainty score is estimated using cosine similarity between the feature vector $(\mathbf F_{t})$ representing the test sample $t$ and first singular vector of each ID class. Here,  BN represents Batch Normalization, $L$ is the number of classes, and $\delta^{th}$ is the threshold for the uncertainty score.}}
\label{Methodfig}
\vspace{-2mm}
\end{figure*}
\vspace{-2mm}

\indent In summary, we make the following contributions in this study. First, we propose that OOD detection test can be designed using the features extracted by self-supervised contrastive learning that reinforce the uni-dimensional projections of the ID set. 
Second, we have theoretically proved that such uni-dimensional projections, boosted by the contrastive learning, can be characterized by the prominent first singular vector that represents its corresponding class attributes.
Furthermore, the robustness of the proposed OOD detector has been evaluated by introducing corruptions in both OOD and ID datasets. Extensive  experiments  illustrate that the proposed OOD detection method outperforms the state-of-the-art (SOTA) algorithms.

\vspace{-2mm}
\section{Approach}
\label{sec:apprach}
\vspace{-2mm}
Our proposed OOD detection approach builds upon employing a  self-supervised training block to extract robust features from the ID dataset. This is carried out by training a contrastive loss on ID data as shown in Fig. \ref{Methodfig} (a). Next, we utilize the concept of \emph{union of one-dimensional-embeddings} to project the deep features of different classes onto one-dimensional and mutually orthogonal predefined vectors representing each class to obtain logits. At the final layer's output, we evaluate the cross-entropy between the logit probability output and the labels to form the supervised loss as shown in Fig. \ref{Methodfig} (b). The uni-dimensional mapping is carried out to guarantee that intra-class distribution consists of samples aligning the most with the uni-dimensional vector characterizing its samples. To this end, the penultimate layer of the model is modified by using cosine similarity and introducing a sharpening layer as shown in Fig. \ref{Methodfig} (b), where output logits are calculated as, $\small{{P}(\mathbf{F}_n) = {Z(\mathbf{F}_n) \over G(\mathbf{F}_n)}}$, where
\vspace{-3mm}
\begin{equation}
\small{
    Z(\mathbf{F}_n)= {\mathbf W^{T}\mathbf {F}_n \over  \lVert \mathbf {F}_n\lVert},
G(\mathbf{F}_n) =\sigma (BN(\mathbf W^{T}_{g} \mathbf F_{n}))}
\end{equation}

\vspace{-1mm}
\indent Here, $\mathbf  F_n$ represents the encoder output for the training sample $n$, $\sigma$ is the sigmoid function, and $ \mathbf W_{g}$ is the weight matrix for the sharpening layer, represented by $G(\mathbf F_n)$, which essentially maps $\mathbf F_n$ to a scalar value. In the sharpening layer, batch normalization (BN) is used for faster convergence as proposed by \cite{hsu2020generalized}. 
It is worth mentioning that during the fine-tuning stage, we do not calculate the \emph{bias} vector for the penultimate and sharpening layers. \\
\indent 
The \emph{orthogonality} comes with wide angles between the uni-dimensional embeddings of separates classes creating a large and expanded rejection region for the OOD samples if they lie in the vast inter-class space. To achieve this, we initialize the weight matrix $\mathbf W=[\mathbf w_l \mathbf w_2 \ldots \mathbf w_l]$ of the penultimate layer with orthonormal vectors as in \cite{saxe2013exact} and then freeze it during the fine-tuning stage. Here, $\mathbf w_l$ represents the weights of the last fully connected layer corresponding to class $l$. During fine-tuning, the features are projected onto the predefined set of orthogonal vectors $\mathbf w_l$ for $l=1,2,\ldots, L$, where $L$ is the number of ID classes.\\ 
\indent After training, OOD testing can be done by evaluating the inner products between the calculated first singular vectors  ($\mathbf {U_1,U_2,\ldots,U_L}$) representing their corresponding classes as shown in Fig. \ref{Methodfig} (c), and the extracted feature for the sample of interest. To perform OOD inspection on the test sample $t \in S_t$, where $S_t$ is the test set, the uncertainty score is calculated as,
\vspace{-2mm}
\begin{equation}
\small{ {\delta_t} = {\min (\arccos \left( {{\mathbf F_{t}^T \mathbf U_l} \over {\lVert \mathbf F_t \Vert}}\right))}, ~~\forall  ~~l \in \{1,2,\ldots,L\}   }
\vspace{-2mm}
\end{equation}

\indent Here, $\mathbf  F_t$ is the output of the encoder for the test sample $t$. The measured uncertainty is then used to calculate the probability that if $t$ belongs to ID or OOD using the probability function $p(\delta_t \leq\delta^{Th}|t \in S_t)$ as \texttt{RODD} is a probalistic approach where sampling is performed during the test time. In an ideal scenario, features of ID class $l$ have to be aligned with the corresponding $\mathbf w_l$, where  $\mathbf w_l$ is the $l^{th}$ column of matrix $\mathbf W$. In that case, $\delta^{Th}=0$. However, in practice, all class features are not  exactly aligned with their respective column in $\mathbf W$,  that further strengthens the idea of using the first singular vector of each class feature matrix, separately. \\
\indent Next, we will explain how the contrastive learning pre-training and sharpening module, $G(\mathbf F_n)$, boosts the performance of our approach. Firstly, contrastive learning has been beneficial because  we do not freeze the weights of the encoder after the self-supervised learning  and keep fine-tuning them along the training procedure using the cross-entropy loss. In other words, the features are warm-started with initialized values derived from the contrastive loss pre-training, yet the final objective function to optimize is composed of two terms $\mathcal{L}_{CL}+\mu \mathcal{L}_{LL},$
where $\mathcal{L}_{CL}$ and $\mathcal{L}_{LL }$ denote the contrastive and cross-entropy losses, respectively. In addition, the cross-entropy loss imposes the orthogonality assumption infused by the choice of orthogonal matrix containing union of  $\mathbf w_l~ \forall~  l \in \{1,2,\ldots,L\}$ each of which represent one class. By feeding the inner products of features with $\mathbf{W}$ into $\mathcal{L}_{LL}$, the features are endorsed to get reshaped to satisfy orthogonality and rotate to align $ \mathbf w_l$.\\ \indent Furthermore, augmenting the data of each class with the adversarial perturbations can improve classification perfromance on ID perturbed data while still detecting the OOD data~\cite{chen2020adversarial,kim2020adversarial}.
Moreover, prior to feeding the optimizer with the inner products for supervised training, we modify the uni-dimensional mappings using $G(\mathbf F_n)$ to optimally benefit from the self-supervised learned features. To compensate for the uni-dimensional confinement which can downgrade the classifier's performance, we use the \emph{sharpening} concept, where we enhance the confidence of the obtained logit vector by scaling the inner products with a factor denoted with the sharpening function $G(\mathbf{F}_n)$ explained above. 

\subsection {Theoretical Analysis }
\label{theory}
\vspace{-2mm}
In this section, we provide theoretical analyses on how pre-training with contrastive loss promotes the uni-dimensional embeddings approach utilized in \texttt{RODD}
by promoting one prominent singular vector (with a dominant singular value) in the deep feature extraction layer.\\
 \indent The objective function used in our optimization is composed of a contrastive loss and a softmax cross entropy. For simplicity, we use a least squared loss measuring the distance between linear prediction on a sample's extracted feature to its label vector $ \Vert \mathbf{W}^T\mathbf{F}_n-\mathbf{y}_n \Vert_2^2$ as a surrogate for the softmax cross entropy ($\mathcal{L}_{LL}$) \footnote{The least squared loss ($\mathcal{L}_{LL}$) measures the distance of the final layer predictions (assuming linear predictor in the deep feature space) from the one-hot encoded vector (alternatively logits if available)}.  This is justified in \cite{xue2022investigating}. 

Let $\mathbf{A}=[a_{i,j}]$ denote the adjacency matrix for the augmentation graph of training data formally defined as in \cite{xue2022investigating}.
 In general, two samples are connected through an edge on this graph if they are believed to be generated from the same class distribution. Without loss of generality, we assume that the adjacency matrix is block-diagonal, i.e., different classes are well-distinguished. Therefore, the problem can be partitioned into data specific to each class. Let $\mathbf{F}$ and $\mathbf{Y}$ denote the matrix of all features and label vectors, i.e., $\mathbf{F}_n$ and $\mathbf{y}_n$, where $n$ denotes the $n^{\text{th}}$ sample, respectively. 
The training loss including one term for contrastive learning loss and one for the supervised uni-dimensional embedding matching can be written as: \footnote{ It is shown in~\cite{haochen2021provable} that the solution to the contrastive learning loss can be written as the following Cholesky decomposition problem, $\min_{\mathbf{F}}\Vert \mathbf{A}- \mathbf{F}\mathbf{F}^T\Vert_F^2$, which constitutes the first term of the loss in Eq. \eqref{eq:the}.}
\begin{equation}
\small{
     \mathcal{L}(\mathbf{F})= 
    \underbrace{\Vert \mathbf{A}- \mathbf{F}\mathbf{F}^T\Vert_F^2}_{\mathcal{L}_{CL}(\mathbf{F})}+\mu \underbrace{\Vert \mathbf{W^T \mathbf{F}}  - \mathbf{Y}\Vert_F^2}_{\mathcal{L}_{LL}(\mathbf{F})}. 
}
\label{eq:the}
\vspace{-2mm}
\end{equation}

$\mathbf{Y}$ and $\mathbf{A}$ are given matrices, and $\mathbf{W}$ is fixed to some orthonormal predefined matrix. The optimization variable is therefore the matrix $\mathbf{F}$.     
Thus, the optimization problem can be written as:
\begin{equation}
\vspace{-2mm}
\small{
    \min_{\mathbf{F}}~~\Vert \mathbf{A}- \mathbf{F}\mathbf{F}^T\Vert_F^2+\mu \Vert \mathbf{W^TF-Y}\Vert_F^2}.
\label{prob:main}
\end{equation}
\vspace{-2mm}

Before bringing the main theorem, two assumptions are made on the structure of the adjacency matrix arising from its properties ~\cite{xue2022investigating}:
\textbf{1}:
For a triple of images $\mathbf{x}_i,\mathbf{x}_j,\mathbf{x}_s$, we have $\frac{a_{i,j}}{a_{j,s}} \in [\frac{1}{1+\delta},1+\delta]$ for small $\delta$, i.e., samples of the same class are similar.
\textbf{2:}
For a quadruple of images $\mathbf{x}_i,\mathbf{x}_j,\mathbf{x}_s, \mathbf{x}_t$, where $\mathbf{x}_i,\mathbf{x}_j$ are from different classes and  $\mathbf{x}_s,\mathbf{x}_t$ are from the same classes,  $\frac{a_{i,j}}{a_{s,t}} \leq \eta$ for small $\eta$.
\vspace{-1mm}
\begin{lemma}
Let $\mathbf{F}^*$ denote the solution to $\min_{\mathbf{F}} \mathcal{L}_{CL}$ (first loss term in \eqref{prob:main}). Assume $\mathbf{F}^*$ can be decomposed as
$\mathbf{F}^*=\mathbf{U}\mathbf{\Sigma} \mathbf{V}^T$.
Under Assumptions 1,2 (above), for $\mathbf{F}^*$ with singular values $\sigma_i$, we have $\sum_{i=2}^{N_l} \sigma_i^2 \leq \sqrt{6 \big ( (1+\delta)^\frac{3}{2}-1 \big)}$ for some small $\delta$, where $\sigma_i=\mathbf{\Sigma}_{ii}$, and $N_l$ is the number training samples of class $l$.
\label{lem:1}
\end{lemma}
\begin{proof}
\vspace{-2mm}
In \cite{xue2022investigating}, it is shown that $\sum_{i=2}^{N_l} \sigma_i^4 \leq 2\big((1+\delta)^\frac{3}{2} -1) \big)$. 
The proof is straightforward powering $\sum_{i=2}^{N_l} \sigma_i^2$ by two and applying Cauchy-Schwartz inequality.

\vspace{-2mm}
\end{proof}
\vspace{-6mm}
\begin{theorem}
Let $\mathbf{F}^*$ denote the solution to ~\eqref{prob:main}. Assume $\mathbf{F}^*$ is decomposed as
$\mathbf{F}^*=\mathbf{U}\mathbf{\Sigma} \mathbf{V}^T$.
There exist a $\mu_{min}$ such that  $\sum_{i=2}^{N_l} \sigma_i^4 \leq 2\big((1+\delta)^\frac{3}{2} -1) \big)$, if $\mu<\mu_{min}$ in P \eqref{prob:main}.
\end{theorem}
\indent
The purpose is to show that treating corrupted or adversarial ID data vs. OOD data, the uni-dimensional embedding is robust in OOD rejection. This mandates invariance and stability of the first singular vector for the features extracted for samples generated from each class. 
The goal of this theorem is to show that using the contrastive loss along certain values of $\mu$ regularizing the logit loss, the dominance of the first eigenvector of the adjacency matrix is also inherited to the first singular vector of the $\mathbf{F}$ and this is inline with the mechanism of proposed approach whose functionality depends on the stability and dominance of the first singular vector because we desire most of the information included in the samples belonging to each class can be reflected in uni-dimensional projections. 

Assuming the dominance is held for the first singular value of each class data, the contrastive learning can therefore split them by summarizing the class-wise data into uni-dimensional separate representations. 
The $\mathbf{V}$ matrix is used to orthogonalize and rotate the uni-dimensional vectors obtained by contrastive learning to match the pre-defined orthogonal set of vectors $\mathbf{w}_l$ as much as possible. 

Now the proof for the main theorem is provided.
\vspace{-2mm}
\begin{proof}
$\mathbf{A}$ is Hermitian. Therefore, it can be decomposed as $\mathbf{A}=\mathbf{Q}\mathbf{\Lambda}\mathbf{Q}^T$.
The solution set to minimize $\mathcal{L}_{CL}$ is $\mathcal{S}=\{ \mathbf{Q}\mathbf{\Lambda}^{\frac{1}{2}} \mathbf{V}^T: \forall \texttt{~orthonormal matrix~} \mathbf{V} \}$ ($\lambda_i=\mathbf{\Lambda}_{ii}=\sigma_i^2$).\\ 
Let $L_1$ and $L_2$ be the minima for ~\eqref{prob:main} obtained on the sets $\mathcal{S}$ and $\mathcal{S}^c$, i.e., the complementary set of $\mathcal{S}$. $L_1$ equals $\mu \min_{\mathbf{F}\in \mathcal{S}}~\mathcal{L}_{LL}(\mathbf{F})$ as the first loss is 0 for elements in $\mathcal{S}$. \\
Now, we consider $L_2$. $\mathcal{S}^c$ can be partitioned into two sets $\mathcal{S}^c_1$ and  $\mathcal{S}^c_2$, where elements in  $\mathcal{S}^c_1$ set $\mathcal{L}_{LL}$ to zero and elements in $\mathcal{S}^c_2$ yield non-zero values for $\mathcal{L}_{LL}$. Therefore, $L_2$ is the minimum of the two partition's minima.
\vspace{-3mm}
\begin{equation}
    L_2=\min \big \{\underbrace{\min_{\mathbf{F} \in \mathcal{S}^c_1}~\mathcal{L}_{CL}(\mathbf{F})}_{LHS},\underbrace{\min_{\mathbf{F} \in \mathcal{S}^c_2}~\mathcal{L}_{CL}(\mathbf{F})+\mu \mathcal{L}_{LL}(\mathbf{F})}_{RHS} \big\}
    \label{eq:L2}
\end{equation}

\indent It is obvious that for a small enough $\mu$, $L_2$ equals the RHS above. This can be reasoned as follows. Let the LHS value be denoted with $m_1$. $m_1>0$ since $\mathcal{S}$ and $\mathcal{S}_1^c$ are disjoint sets with no sharing boundaries. 
The RHS in \eqref{eq:L2} is composed of two parts. The first part can be arbitrarily small because although $\mathcal{S}$ and $\mathcal{S}_2^c$ are disjoint, they are connected sets with sharing boundaries. (For instance any small perturbation in $\mathbf{\Lambda}$ eigenvalues drags a matrix from $\mathcal{S}$ into $\mathcal{S}_2^c$. However, they are infinitesimally close due to the continuity property). The second term can also be shrunk with an arbitrarily small choice of $\mu=\mu_{min}=\frac{m_1}{\mathcal{L}_{LL}(\tilde{\mathbf{F}})}$ that guarantees the RHS takes the minimum in Eq. \eqref{eq:L2}, where $\tilde{\mathbf{F}} = \underset{\mathbf{F} \in \mathcal{S}_2^c} {\argmin}~\mathcal{L}_{CL}(\mathbf{F})$ \footnote{(As discussed, $\Tilde{\mathbf{F}}$ makes the first term arbitrarily approach 0 due to continuity property holding between $\mathcal{S}$ and $\mathcal{S}_2^c$ and there is an element in $\mathcal{S}_2^c$ arbitrarily close to $\tilde{\mathbf{F}}$)}.
Therefore, for $\mu<\mu_{min}$, the minimum objective value in Eq. \eqref{prob:main} ($\min\{L_1,L_2\}$) is,
$
\min \big\{ \min_{\mathbf{F} \in \mathcal{S}^c_2}~\mathcal{L}_{CL}(\mathbf{F})+\mu \mathcal{L}_{LL}(\mathbf{F}), \min_{\mathbf{F} \in \mathbf{S}}~ \mu \mathcal{L}_{LL}(\mathbf{F})  \big\} 
\label{eq:min}$.
\indent 
The final aim is to show that $\mu$ can be chosen such that $\mathbf{F}^*$ inherits the dominance of first eigenvalue from $\mathbf{A}$. This is straightforward if the solution is RHS in \eqref{eq:L2} because the solution lies on $\mathcal{S}$ in that case and therefore, can be expressed as $\mathbf{Q}\mathbf{\Lambda}^\frac{1}{2}\mathbf{V}^T$ inheriting the property in Lemma~\ref{lem:1}.

\indent Thus, we first consider cases where $\min \{L_1,L_2\}$ is obtained by the RHS by explicitly writing when LHS$>$RHS.\\
We assume the minimizers for the RHS and LHS differ in a matrix $\mathbf{R}$. Let $\mathbf{F}^*$ denote the minimizer for RHS. Then, the minimizer of LHS is $\mathbf{F}^*+\mathbf{R}$.
We have \\
$\text{LHS}= \Vert \mathbf{A}-(\mathbf{F^*}+\mathbf{R})(\mathbf{F^*}+\mathbf{R})^T \Vert_F^2+ \mu \Vert \mathbf{W}^T \mathbf{F}^* + 
      \mathbf{W}^T\mathbf{R}-\mathbf{Y}\Vert_F^2 = \\
 {\Vert \underbrace{\mathbf{A}-\mathbf{F}^*\mathbf{F}^{*T}}_{0}}-\underbrace{(\mathbf{F}^*\mathbf{R}^T+\mathbf{R}\mathbf{F}^{*T}+\mathbf{RR}^T)}_{\mathbf{E}}\Vert_F^2+\mu \Vert \mathbf{W}^T\mathbf{F}^*-\mathbf{Y}+\mathbf{W}^T\mathbf{R} \Vert_F^2 = 
   \Vert \mathbf{E} \Vert _F^2 +\mu \Vert \mathbf{W}^T \mathbf{F}^*-\mathbf{Y}\Vert_F^2+\mu \Vert \mathbf{W}^T\mathbf{R}\Vert_F^2+2\mu \langle \mathbf{W}^T\mathbf{F}^*-\mathbf{Y}, \mathbf{W}^T\mathbf{R} \rangle,
   $\\
where the inner product of two matrices $\mathbf{A,B}$ ($\langle \mathbf{A,B}\rangle$) is defined as $Tr(\mathbf{AB}^T)$.
The RHS in~\eqref{eq:L2} equates $\mu \Vert \mathbf{W}^T\mathbf{F}^*-\mathbf{Y} \Vert_F^2$ since $\mathbf{F}^*$ is its minimizer and the loss has only the logit loss term. Thus, the condition $\small{\text{LHS}>\text{RHS}}$ reduces to 
$\small{
\Vert \mathbf{E} \Vert _F^2 + \mu \Vert \mathbf{W}^T\mathbf{R}\Vert_F^2+2\mu\langle \mathbf{W}^T\mathbf{F}^*-\mathbf{Y}, \mathbf{W}^T\mathbf{R} \rangle >0 
}$.
\vspace{-1mm}
Using the fact that the matrix $\mathbf{W}$ is predefined to be an orthonormal matrix, multiplying it by $\mathbf{R}$ does not change the Frobenius norm. Hence, the condition reduces to
$
\Vert \mathbf{E} \Vert _F^2 + \mu \Vert \mathbf{R}\Vert_F^2 > 2\mu \langle \mathbf{Y}-\mathbf{W}^T\mathbf{F}^*, \mathbf{W}^T\mathbf{R} \rangle 
$. 
\indent To establish this bound, the Cauchy-Schwartz inequality (C-S) and the Inequality of Arithmetic and Geometric Means (AM-GM) are used to obtain the upper bound for the inner product. The sufficient condition holds true if it is established for the obtained upper bound (tighter inequality).
Applying (C-S) and (AM-GM) inequalities we have 
\vspace{-3mm}
\begin{align*}
{\small{
    \begin{array}{cc}
\langle \mathbf{Y}-\mathbf{W}^T\mathbf{F}^*, \mathbf{W}^T\mathbf{R} \rangle \overset{\overset{C-S}{\rightarrow}}{\leq} \Vert \mathbf{Y}-\mathbf{W}^T\mathbf{F}^* \Vert_F \Vert \mathbf{W}^T\mathbf{R} \Vert_F=\\ \Vert \mathbf{Y}-\mathbf{W}^T\mathbf{F}^* \Vert_F \Vert \mathbf{R} \Vert_F  \overset{\overset{AM-GM}{\rightarrow}}{\leq} \frac{1}{2} \Vert \mathbf{Y}-\mathbf{W}^T\mathbf{F}^* \Vert_F^2 + \frac{1}{2} \Vert \mathbf{R} \Vert _F^2
    \end{array}
    \vspace{-4mm}
    }}
\end{align*}
Substituting this for the inner product to establish a tighter inequality, we get
$
\Vert \mathbf{E} \Vert _F^2 + \mu \Vert \mathbf{R}\Vert_F^2 > \mu \Vert \mathbf{Y}-\mathbf{W}^T\mathbf{F}^* \Vert_F^2 + \mu \Vert \mathbf{R} \Vert _F^2 
$
reducing to
$\Vert \mathbf{E} \Vert _F^2 > \mu \Vert \mathbf{Y}-\mathbf{W}^T\mathbf{F}^* \Vert_F^2$.\\
\indent As the matrix of all zeros, i.e., $[\bf{0}] \in \mathcal{S}$, inserting $[\bf{0}]$ for $\mathbf{F}$ leads to a trivial upper bound for the minimum obtained over $\mathbf{F} \in \mathcal{S}$, i.e., $\Vert \mathbf{Y}-\mathbf{W}^T\mathbf{F}^*\Vert_F^2$ is upper bounded with $\Vert \mathbf{Y} \Vert_F^2$. Finding a condition for $\Vert \mathbf{E} \Vert _F^2 > \mu_{min} \Vert \mathbf{Y} \Vert_F^2$ guarantees the desired condition is satisfied.
If $\Vert \mathbf{E} \Vert _F^2 > \mu_{min} \Vert \mathbf{Y} \Vert_F^2$ is met, the solution lies in $\mathcal{S}$ and RHS obtains the minimum, validating Lemma \ref{lem:1} for $\mathbf{F}^*$.\\
Otherwise, if the solution lies in $\mathcal{S}_2^c$ and is  attained from the LHS such that it contravenes the dominance of the first pricinpal component of $\mathbf{A}$, we will show by contradiction that the proper choice for $\mu$ avoids LHS to be less than the RHS in~\eqref{eq:L2}. To this end, we take a more profound look into $\Vert \mathbf{E} \Vert_F^2$.
If $\mathbf{R}$ is to perturb the solution $\mathbf{F}^*$ such that the first principal component is not prominent,  for $\mathbf{R}+\mathbf{F}^*$, we shall have $\sum_{i=2}^{N_l}~\sigma_i^2 > \Delta + \alpha$ for some positive $\alpha$ violating the condition stated in the Theorem. 
This means there is at least one singular value of $\mathbf{F}^*+\mathbf{R}$, for which we have $\sigma_r>\sqrt{\frac{\Delta+\alpha}{N_l-1}}= \sqrt{\frac{\alpha}{N_l-1}}+\mathcal{O}(\sqrt[4]{\delta})$.
As $\mathbf{F}^*$ inherits the square root of eigenvalues of $\mathbf{A}$, according to Lemma \ref{lem:1} and using Taylor series expansion, $\sigma_r(\mathbf{F}^*)=\mathcal{O}(\sqrt[4]{\delta})$. This yields $\sigma_r(R)> \sqrt{\frac{\alpha}{N_l-1}}+\mathcal{O}(\sqrt[4]{\delta}).$
$\mathbf{E}$ is a symmetric matrix and therefore it has eigenvalue decomposition.
$
 \Vert \mathbf{E} \Vert_F^2 \geq \lambda_r^2(\mathbf{E}) =  \lambda_r^2(\mathbf{RR}^T+\mathbf{RF}^{*T}+\mathbf{F}^*\mathbf{R}^T)= \\  \lambda_r^2(\mathbf{RR}^T)+\mathcal{O}(\delta) > \frac{\alpha^2}{(N_l-1)^2}+\mathcal{O}(\delta).
$
Knowing that $\Vert \mathbf{Y} \Vert_F^2 = N_l^2$, if $\mu < \frac{\alpha^2}{N_l^4}$, the condition for RHS$<$LHS is met.
According to Lemma \ref{lem:1} and the previous bound found for $\mu_{min}$, if $\mu_{min} < \min \{\frac{\alpha^2}{N_l^4},\frac{m_1}{\mathcal{L}_{LL}(\tilde{\mathbf{F}})}\}$, the solution should be $\mathbf{F}^*=\mathbf{Q\mathbf{\Lambda}}^\frac{1}{2}\mathbf{V}^T$.
Hence, for certain range of values for $\mu$, the solution takes the form $\mathbf{Q}\mathbf{\Lambda}^{\frac{1}{2}}\mathbf{V}$ obeying the dominance of $\lambda_1$ in $\mathbf{A}$ and this concludes the proof. 
\end{proof}
\vspace{-3mm}
\begin{table*}[h!]
       \caption{\small OOD detection results of {\texttt{RODD}} and comparison with competitive baselines trained on CIFAR-10 as ID dataset. All values are shown in percentages. $\uparrow$ indicates larger values are better and $\downarrow$ indicates smaller values are better.}
\centering
\scalebox{0.7}{
\begin{tabular}{lcccccccccccccccc}
\toprule
\multicolumn{1}{c}{\multirow{4}{*}{\textbf{Methods}}} & \multicolumn{14}{c}{\textbf{OOD Datasets}} \\ \cline{2-15}
\multicolumn{1}{c}{}  & \multicolumn{2}{c}{\textbf{SVHN}} & \multicolumn{2}{c}{\textbf{iSUN}}& \multicolumn{2}{c}{\textbf{LSUNr}}&\multicolumn{2}{c}{\textbf{TINc}}& \multicolumn{2}{c}{\textbf{TINr}} & \multicolumn{2}{c}{\textbf{Places}} & \multicolumn{2}{c}{\textbf{Textures}} \\
\multicolumn{1}{c}{}  & FPR95 & AUROC & FPR95 & AUROC& FPR95 & AUROC& FPR95 & AUROC& FPR95 & AUROC& FPR95 & AUROC & FPR95 & AUROC\\
\multicolumn{1}{c}{} & \multicolumn{1}{c}{$\downarrow$} & \multicolumn{1}{c}{$\uparrow$} & \multicolumn{1}{c}{$\downarrow$} & \multicolumn{1}{c}{$\uparrow$} & \multicolumn{1}{c}{$\downarrow$} & \multicolumn{1}{c}{$\uparrow$} & \multicolumn{1}{c}{$\downarrow$} & \multicolumn{1}{c}{$\uparrow$} & \multicolumn{1}{c}{$\downarrow$} & \multicolumn{1}{c}{$\uparrow$}& \multicolumn{1}{c}{$\uparrow$} & \multicolumn{1}{c}{$\downarrow$}& \multicolumn{1}{c}{$\uparrow$} & \multicolumn{1}{c}{$\downarrow$} \\
\midrule
MSP ~\cite{Kevin} & 48.49 & 91.89 & 56.03 & 89.83 & 52.15 & 91.37 & 53.15 & 87.33 & 54.24  & 79.35 & 59.48 & 88.20 & 59.28 & 88.50\\
 ODIN ~\cite{liang2017enhancing} & 33.55 & 91.96 & 32.05 & 93.50 & 26.52   & 94.57 &36.75 &89.20 &49.15 &81.64& 57.40 & 84.49 & 49.12 & 84.97 \\
 Mahalanobis \cite{lee2018simple} & 12.89 & 97.62 & 44.18 & 92.66 & 42.62 & 93.23 & 42.75 & 88.85  & 52.25  & 80.33 & 92.38 & 33.06 & 15.00 & 97.33 \\
 Energy ~\cite{liu2020energy} & 35.59 & 90.96 & 33.68 & 92.62 & 27.58 & 94.24 & 35.69  & 89.05 & 50.45  & 81.33 & 40.14 & 89.89 & 52.79 & 85.22  \\
OE ~\cite{hendrycks2018deep} & 4.36 & 98.63 & 6.32 & 98.85 & 5.59 & 98.94 & 13.45 & 96.44 & 15.67 & 96.78 & 19.07 & 96.16 & 12.94 & 97.73 \\
VOS ~\cite{du2022vos} & 8.65 & 98.51 & 7.56 & 98.71 & 14.62 & 97.18 & 11.76 & 97.58 & 28.08 & 94.26 & 37.61 & 90.42 & 47.09 & 86.64 \\
FS ~\cite{zaeemzadeh2021out} & 24.71 & 95.31 & 17.41 & 96.61 & 4.84 & 96.28 & 12.45 & 97.83 & 9.65 & 97.95 & 11.56 & 96.42 & 5.55 & 98.64 \\
\textbf{\texttt{RODD} (Ours)} & \textbf{1.82} & \textbf{99.63} & \textbf{4.07} & \textbf{99.32} & \textbf{4.49} & \textbf{99.25} & \textbf{10.29} & \textbf{98.10 } & \textbf{6.30} & \textbf{99.0 } & \textbf{9.59} & \textbf{98.47 } & \textbf{3.87} & \textbf{99.43}\\
\hline
\vspace{-5mm} 
\end{tabular}
}
\label{tab:table1}
\end{table*}
\vspace{4mm}
\begin{table*}[h!]
       \caption{\small OOD detection results of \texttt{RODD} and comparison with competitive baselines trained on CIFAR-100 as ID dataset. All values are shown in percentages. $\uparrow$\ indicates larger values are better and $\downarrow$ indicates smaller values are better.}
\centering
\scalebox{0.7}{
\begin{tabular}{lcccccccccccccccc}
\toprule
\multicolumn{1}{c}{\multirow{4}{*}{\textbf{Methods}}} & \multicolumn{14}{c}{\textbf{OOD Datasets}} \\ \cline{2-15}
\multicolumn{1}{c}{}  & \multicolumn{2}{c}{\textbf{SVHN}} & \multicolumn{2}{c}{\textbf{iSUN}}&\multicolumn{2}{c}{\textbf{LSUNr}}&\multicolumn{2}{c}{\textbf{TINc}}&\multicolumn{2}{c}{\textbf{TINr}} & \multicolumn{2}{c}{\textbf{Places}} & \multicolumn{2}{c}{\textbf{Textures}} \\
\multicolumn{1}{c}{}  & FPR95 & AUROC & FPR95 & AUROC& FPR95 & AUROC& FPR95 & AUROC& FPR95 & AUROC& FPR95 & AUROC & FPR95 & AUROC\\
\multicolumn{1}{c}{} & \multicolumn{1}{c}{$\downarrow$} & \multicolumn{1}{c}{$\uparrow$} & \multicolumn{1}{c}{$\downarrow$} & \multicolumn{1}{c}{$\uparrow$} & \multicolumn{1}{c}{$\downarrow$} & \multicolumn{1}{c}{$\uparrow$} & \multicolumn{1}{c}{$\downarrow$} & \multicolumn{1}{c}{$\uparrow$} & \multicolumn{1}{c}{$\downarrow$} & \multicolumn{1}{c}{$\uparrow$}& \multicolumn{1}{c}{$\uparrow$} & \multicolumn{1}{c}{$\downarrow$}& \multicolumn{1}{c}{$\uparrow$} & \multicolumn{1}{c}{$\downarrow$} \\
\midrule
MSP ~\cite{Kevin} & 84.59 & 71.44 & 82.80 & 75.46 & 82.42 & 75.38 & 69.82 &79.77 &79.95  &72.36 & 82.84 & 73.78 & 83.29 & 73.34\\
 ODIN ~\cite{liang2017enhancing} & 84.66 & 67.26 & 68.51 & 82.69 & 71.96   & 81.82 &45.55 &87.77 &57.34 &80.88& 87.88 & 71.63 & 49.12 & 84.97 \\
 Mahalanobis \cite{lee2018simple} & 57.52 & 86.01 & 26.10 & 94.58 & 21.23 & 96.00 & 43.45  &86.65    &44.45&85.68 & 88.83 & 67.87 & 39.39 & 90.57 \\
 Energy ~\cite{liu2020energy} & 85.52 & 73.99 & 81.04 & 78.91 & 79.47 & 79.23 & 68.85  &78.85 &77.65  & 74.56 & 40.14 & 89.89 & 52.79 & 85.22  \\
OE ~\cite{hendrycks2018deep} & 65.91 & 86.66 & 72.39 & 78.61 & 69.36 & 79.71 &46.75 &85.45 &78.76 &75.89 & 57.92 & 85.78 & 61.11 & 84.56 \\
VOS ~\cite{du2022vos} & 65.56 & 87.86 & 74.65 & 82.12 & 70.58 & 83.76 & 47.16 & 90.98 & 73.78 & 81.58 & 84.45 & 72.20 & 82.43 & 76.95 \\
FS ~\cite{zaeemzadeh2021out} & 22.75 & 94.33 & 45.45 & 85.61 & 40.52 & 87.21 & \textbf{11.76} & \textbf{97.58} & 44.08 & 86.26 & 47.61 & 88.42 & 47.09 & 86.64 \\
\textbf{\texttt{RODD} (Ours)} & \textbf{19.89} & \textbf{95.76} & \textbf{39.79} & \textbf{88.40} & \textbf{36.61} & \textbf{89.73} & 44.42  & 85.95 & \textbf{42.56} & \textbf{87.67 } & \textbf{41.72} & \textbf{89.10 }& \textbf{24.64} & \textbf{94.14} \\
\hline
\end{tabular}
}
\label{tab:table2}
\end{table*}
\begin{figure*}[h!]
\centering
\includegraphics[height=5cm, width=15cm,trim={5.3cm 7cm 6.5cm 4cm}, clip]{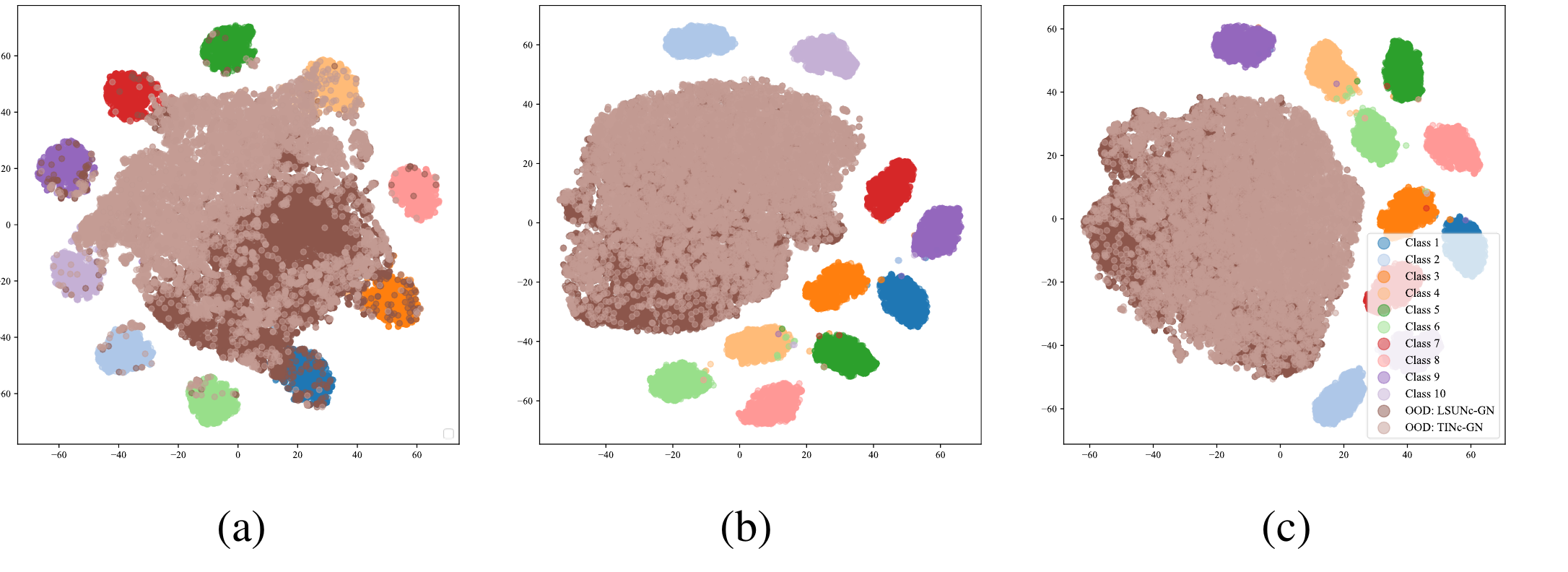}
 \vspace{-3mm}
\caption{\footnotesize t-SNE representation of features extracted by introducing Gaussian noise on OOD dataset. 10,000 samples each of TINc and LSUNc while 1,000 sample of each class from ID CIFAR-10 test set are used to generate 2D t-SNE plot. \textcolor{red}{(a)} Features extracted from the baseline model with severity level 1.
\textcolor{red}{(b)} Features extracted using \texttt{RODD} with corruption severity level 1. 
\textcolor{red}{(c)} Features extracted from the \texttt{RODD} with corruption severity level 5. 
}\label{tsne}
\end{figure*}
\begin{table*}[h!]
				\vspace{-1mm}
		\caption{\footnotesize Evaluation using corrupted ID test samples for CIFAR-100. All values are in $\%$  and averaged over 7 OOD datasets discussed in Section \ref{datasets} whereas corruption severity is varied from 1-5 as in \cite{hendrycks2019robustness}. $\uparrow$  indicates larger values are better and $\downarrow$ indicates smaller values are better. }
		\vspace{-3mm}
		\centering
		\scalebox{0.645}{
			{
				\begin{tabular}{l|c|c|ccc|ccc|cccc|cccc}
					\hline
	
					\multirow{2}{*}{Dataset}             & \multirow{2}{*}{Method} & \multirow{2}{*}{Clean} &              \multicolumn{3}{c|}{Noise}               &                     \multicolumn{3}{c|}{Blur}                     &                   \multicolumn{4}{c|}{Weather}                    & \multicolumn{4}{c}{Digital}                                    \\
					&                         &                        &       Gauss       &       Shot       &    Impulse     &    Defocus     &         Motion     &      Zoom      &      Snow      &     Frost      &      Fog       &     Bright     &    Cont.    &    Elastic     &     Pixel      &      JPEG       \\ \hline
					
					\multirow{2}{*}{$\downarrow$FPR95}  &           VOS           &         66.79                 &      72.55       &     76.95      &     90.36            &     84.50      &     83.62      &     84.56      &     87.0      & 83.34 & 83.84 & 86.11 &     86.67      &     85.81      &     89.58  &89.25            \\
					&     RODD      &   \bf{39.76}     &      \bf{67.91} & \bf{65.42}&\bf{65.53}& \bf{49.51}& \bf{71.81}& \bf{55.87}& \bf{53.92}& \bf{59.84}&\bf{ 52.23}& \bf{48.39}&\bf{ 52.98}&\bf{ 57.31}&\bf{ 55.42}& \bf{66.47}       \\
					 \hline
					\multirow{2}{*}{$\uparrow$AUROC} &           VOS          &     81.9     &      74.26        &      72.90       &     60.00      &      68.35            &     69.83      &      68.55      &      65.31      &     68.14      &     68.50            &      66.54      &     66.82      &     66.98      &     61.18      &  62.38        \\
					&     RODD      &        \bf{ 88.1 }         & \bf{ 77.18 } &\bf{78.40} &  \bf{78.41}  & \bf{84.70} & \bf{74.64}  &\bf{ 82.42} &     \bf{83.50 }     & \bf{80.60} & \bf{83.85} & \bf{85.54} & \bf{83.44} &      \bf{ 81.91}      &   \bf{  83.11 }     &  \bf{  78.19}        
					       \\ \hline
		\end{tabular}}}
		\label{tab:table3}
	\end{table*}
\begin{table*}[h!]
		 		\vspace{-2mm}
		\caption{\footnotesize Clean and corruption accuracy (\%) of \texttt{RODD} and Baseline  on   CIFAR10-C and CIFAR100-C.}
		\label{tabel7}
		\centering
		\scalebox{0.645}{
			{
				\begin{tabular}{l|c|c|ccc|ccc|cccc|cccc}
					\hline
	
					\multirow{2}{*}{Dataset}             & \multirow{2}{*}{Method} & \multirow{2}{*}{Clean} &              \multicolumn{3}{c|}{Noise}               &                     \multicolumn{3}{c|}{Blur}                     &                   \multicolumn{4}{c|}{Weather}                    & \multicolumn{4}{c}{Digital}                                    \\
					&                         &                        &       Gauss       &       Shot       &    Impulse     &    Defocus     &         Motion     &      Zoom      &      Snow      &     Frost      &      Fog       &     Bright     &    Cont.    &    Elastic     &     Pixel      &      JPEG       \\ \hline
					
					\multirow{2}{*}{CIFAR10-C}  &           Baseline           &         \textbf{94.52}          &       46.54       &      57.72       &     \textbf{56.45}      &     69.15            &     62.98      &     58.85      &     74.88      &     72.18      & 84.26 & 92.19 & \textbf{75.14} &     74.31      &     68.27      &     77.34              \\
					&     RODD      &   94.45     &     \textbf{  49.63}       &      \textbf{59.89}       &     55.62      &    \textbf{ 69.77}                 &     \textbf{64.81}  & \textbf{61.79}    & \textbf{
					78.59}      &     \textbf{74.48}      &     \textbf{86.56}      &    \textbf {93.08}      &     73.37      &    \textbf{ 75.49}      &     \textbf{70.79}      &    \textbf{80.12}            \\
					 \hline
					\multirow{2}{*}{CIFAR100-C} &           Baseline           &    \textbf{ 72.35 }    &       18.80       &      26.56       &     25.56      &      49.80            &     40.45      &      39.37      &      45.38      &     42.62      &     56.40            &      69.14      &     \textbf{52.87}      &     48.32      &     40.70      &  46.11        \\
					&     RODD      &         72.20          &  \textbf{18.40}   &  \textbf{27.13}  & \textbf{26.25} & \textbf{50.32}       & \textbf{41.82} &    \textbf{40.40}      & \textbf{46.25} & \textbf{43.46} & \textbf{57.13} & \textbf{70.0} & 51.81 &     \textbf{49.05  }    &    \textbf{ 40.86}      &     \textbf{47.62 }       
					       \\ \hline
		\end{tabular}}}
		\label{tab:table4}
		\vspace{-5mm}
	\end{table*}
\vspace{-4mm}
\section{Experiments}\label{experiments}
\vspace{-1mm}
In this section, we evaluate our proposed OOD detection method through extensive experimentation on different ID and OOD datasets with multiple architectures. 
\vspace{-2mm}
\subsection{Datasets and Architecture} \label{datasets}
\vspace{-1mm}
In our experiments, we used CIFAR-10 and CIFAR-100\cite{krizhevsky2009learning} as ID datasets and 7 OOD datasets. 
OOD datasets utilized are TinyImageNet-crop (TINc), TinyImageNet-resize(TINr)\cite{deng2009imagenet}, LSUN-resize (LSUN-r) \cite{yu2015lsun},  Places\cite{zhou2017places}, Textures\cite{cimpoi2014describing}, SVHN\cite{netzer2011reading} and iSUN\cite{xu2015turkergaze}. For an architecture, we deployed WideResNet~\cite{zagoruyko2016wide} with depth and width equal to 40 and 2, respectively, as an encoder in our experiments. However, the penultimate layer has been modified as compared to the baseline architecture as shown in Fig. \ref{Methodfig}.

\vspace{-1mm}
\subsection{Evaluation Metrics and Inference Criterion} \label{metrics}
\vspace{-1mm}
As in \cite{sun2021react,du2022vos}, the OOD detection performance of \texttt{RODD} is evaluated  using the following metrics: {\emph{(i) FPR95}}
indicates the false positive rate (FPR) at $95\%$ true positive
rate (TPR) and {\emph{(ii) AUROC}}, which is defined as the Area Under the Receiver Operating Characteristic curve. 
As \texttt{RODD} is a probabilistic approach, sampling is preformed on the ID and OOD data during the test time to ensure the probabilistic settings. We employ Monte Carlo sampling to estimate  $p (\delta_{t}\leq \delta^{Th}) $ for OOD detection, where $\delta^{Th}$ is the uncertainty score threshold calculated using training samples. During inference, 50 samples are drawn for a given sample, $t$. The evaluation metrics are then applied on ID test data and OOD data using the estimated $\delta^{Th}$ to calculate the difference in the feature space.
\vspace{-1mm}
\subsection{Results} \label{results}
\vspace{-1mm}

We show the performance of \texttt{RODD} in Tables \ref{tab:table1} and \ref{tab:table2} for CIFAR-10 and CIFAR-100, respectively. Our method achieves an FPR95 improvement of 21.66\%, compared to the most recently reported SOTA~\cite{du2022vos}, on CIFAR-10. We obtain similar performance gains for CIFAR-100 dataset as well. For \texttt{RODD}, the model is first pre-trained using self-supervised adversarial contrastive learning \cite{jiang2020robust}. 
We fine-tune the model following the training settings in \cite{zagoruyko2016wide}.

\vspace{-1mm}
\section{Ablation Studies}
\vspace{-1mm}
\label{ablation}
In this section, we conduct extensive ablation studies to evaluate the robustness of \texttt{RODD} against corrupted ID and OOD test samples. Firstly, we apply the 14 corruptions in \cite{hendrycks2019robustness} on OOD data to generate \emph{corrupted} OOD (OOD-C). Corruptions introduced can be benign or destructive based on thier intensity which is defined by their severity level. To do comprehensive evaluations, 5 severity levels of the corruptions are infused. By introducing such corruptions in OOD datasets, the calculated mean detection error for both CIFAR-10 and CIFAR-100 is 0$\%$, which highlights the inherit property of \texttt{RODD} that it shifts perturbed OOD features further away from the ID as shown in \texttt{t-SNE} plots in Fig. \ref{tsne} which shows that perturbing OOD improves the \texttt{RODD}'s performance. Secondly, we introduced corruptions~\cite{hendrycks2019robustness} in the ID test data while keeping OOD data clean during testing. The performance of \texttt{RODD} on corrupted CIFAR-100 (CIFAR100-C) has been compared with \texttt{VOS}\cite{du2022vos} in Table \ref{tab:table3}. Lastly, we compared the classification accuracy of our proposed method  with the baseline WideResNet model \cite{zagoruyko2016wide} on clean and corrupted ID test samples in Table \ref{tab:table4}. \texttt{RODD} has improved accuracy on corrupted ID test data as compared to the baseline with a negligible drop on classification accuracy of clean ID test data.
\vspace{-1mm}
\section{Conclusion}
\label{conclusion}
\vspace{-1mm}
In this work, we have proposed that in-distribution features can be aligned in a narrow region of the latent space using constrastive pre-training and uni-dimensional feature mapping. With such compact mapping, a representative first singular vector can be calculated from the features for each in-distribution class. The cosine similarity between these computed singular vectors and an extracted feature vector of the test sample is then estimated to perform OOD test. We have shown through extensive experimentation that our method achieves SOTA OOD detection results on CIFAR-10 and CIFAR-100 image classification benchmarks.
\vspace{-1mm}
\section{Acknowledgement}
\vspace{-1mm}
This research is based upon work supported by Leonardo DRS and partly by the National Science Foundation under Grant No. CCF-1718195 and ECCS-1810256.\\
\newpage
{\small
\bibliographystyle{ieee_fullname}
\bibliography{main}
}

\end{document}